\pdfoutput=1

\RequirePackage{fix-cm,amsmath}
\makeatletter
\def\cl@chapter{}
\makeatother

\documentclass[smallcondensed]{svjour3}

\smartqed

\usepackage[scaled=0.86]{helvet}
\usepackage[T1]{fontenc}
\usepackage[dvipsnames]{xcolor}
\usepackage{float}
\usepackage[caption=false]{subfig}
\usepackage{graphicx}
\usepackage{placeins}
\usepackage[sort]{natbib}
\usepackage{bibentry}
\usepackage{algorithm}
\usepackage[noend]{algorithmic}
\usepackage{hyperref}
\usepackage{booktabs}
\usepackage{multirow}

\usepackage{amsmath,amsfonts,amsthm,amssymb}
\usepackage[capitalize]{cleveref}
\usepackage[inline,shortlabels]{enumitem}
\usepackage{calc}
\usepackage{siunitx}

\Crefname{equation}{}{}\crefname{equation}{}{}
\Crefname{myprob}{Problem}{Problems}\crefname{myprob}{Problem}{Problems}
\theoremstyle{plain}

\newtheorem{mylemma}{Lemma}
\newtheorem{mythm}{Theorem}
\theoremstyle{definition}

\theoremstyle{remark}
\newtheorem{myremark}{Remark}

\setlist[description]{font=\normalfont}

\newcommand{\tabformathead}[1]{\textbf{\boldmath{#1}}}

\newenvironment{mtable}[1]{\begin{table}[t]\caption{#1}\begin{center}\begin{scriptsize}}{\end{scriptsize}\end{center}\end{table}}

\DeclareMathOperator{\signo}{sign}
\newcommand{\sign}[1]{\signo\prn{#1}}
\let\originalleft\left
\let\originalright\right
\renewcommand{\left}{\mathopen{}\mathclose\bgroup\originalleft}
\renewcommand{\right}{\aftergroup\egroup\originalright}
\newcommand{\prn}[1]{\left ( #1 \right )}
\newcommand{\brc}[1]{\left \{ #1 \right \}}
\newcommand{\brq}[1]{\left [ #1 \right ]}
\newcommand{\set}[1]{\brc{#1}}
\newcommand{\R}{\mathbb{R}}

\newcommand{\inp}[2]{\langle #1 | #2 \rangle}

\newcommand{\sepeq}{ , \quad}

\newcommand{\agropt}[1]{#1}
\newcommand{\sepopt}{~}

\newcommand{\sepcon}{\quad}
\newcommand{\minp}[2]{\min_{#1}\sepopt{\agropt{#2}}}
\newcommand{\minpc}[3]{\min_{#1}\sepopt{\agropt{#2}}\sepcon\text{s.t. }#3}
\newcommand{\minpcl}[3]{\left \{\begin{array}{l}\displaystyle \min_{#1}\sepopt{\agropt{#2}}\\\displaystyle\text{s.t. }#3\end{array} \right .}

\newcommand{\sreg}{\textsf{ZhouGC}}

\newcommand{\sbel}{\textsf{BelkGC}}

\newcommand{\srob}{\textsf{RobustGC}}
\newcommand{\srobf}{\textsf{PF-\srob}}

\newcommand{\includeg}[1]{\includegraphics{./#1.pdf}}
\newcommand{\mycaption}{}
\newenvironment{mfigure}[1]{\renewcommand{\mycaption}{#1}\begin{figure}[tbp]\begin{center}\centering}{\caption[PDFCaption]{\mycaption}\end{center}\end{figure}}
\newcommand{\plotline}[2]{\includegraphics{./Line#1#2.pdf}}
\newcommand{\plotpattern}[2]{\includegraphics{./Pattern#1#2.pdf}}
\newcommand{\showlegend}[1]{[\,\raisebox{0.5\height}{\plotline{}{#1}}\,]}
\newenvironment{malgorithm}[1]{\begin{algorithm}[t]\caption{#1}\begin{center}\begin{footnotesize}}{\end{footnotesize}\end{center}\end{algorithm}}

\newcommand{\algstop}{;}
\newcommand{\algreq}{$\cdot$~}
\algsetup{linenosize=\tiny}
\colorlet{colorbest}{OliveGreen}
\colorlet{colorworst}{OliveGreen!10}
\newcommand{\formatgen}[4]{{\colorbox{colorbest!#1!colorworst}{\color{#2}\makebox[\widthof{$99.9\pm99.9$}][c]{\raisebox{0pt}[\height][0pt]{#3}}$\enspace^{(#4)}$}}}
\newcommand{\formata}[1]{\formatgen{100}{Black}{#1}{1}}
\newcommand{\formatb}[1]{\formatgen{66}{Black}{#1}{2}}
\newcommand{\formatc}[1]{\formatgen{33}{Black}{#1}{3}}
\newcommand{\formatd}[1]{\formatgen{0}{Black}{#1}{4}}

\newcommand{\emptycell}{---}
\newcommand{\emptyrow}{\emptycell & \emptycell & \emptycell & \emptycell & \emptycell \\}
\newcommand{\graph}{\mathcal{G}}
\newcommand{\vers}{\mathcal{V}}
\newcommand{\versl}{\vers_\mathrm{L}}
\newcommand{\edgs}{\mathcal{E}}
\newcommand{\lo}{\ell_1}
\newcommand{\lt}{\ell_2}
\newcommand{\tr}{\intercal}
\DeclareMathOperator{\diago}{diag}
\newcommand{\diag}[1]{\diago\prn{#1}}
\newcommand{\Ln}{L^\mathrm{N}}
\newcommand{\Lnt}{\tilde{L}^\mathrm{N}}
\newcommand{\Doh}{D^{-1/2}}
\newcommand{\evec}[1]{v_{#1}}
\newcommand{\eval}[1]{\lambda_{#1}}
\newcommand{\evmax}[1]{\eval{\max}\prn{#1}}
\newcommand{\evmin}[1]{\eval{\min}\prn{#1}}
\newcommand{\iden}{\mathbb{I}}
\newcommand{\norm}[1]{\left \| #1 \right \|_2}
\newcommand{\normp}[2]{\left \| #1 \right \|_{#2}}
\newcommand{\abs}[1]{\left | #1 \right |}
\newcommand{\opt}{^{\star}}
\newcommand{\maxp}[1]{\max\prn{#1}}
\newcommand{\loss}[1]{\operatorname{L}\prn{#1}}
\newcommand{\pz}[1]{P_0 #1}

\begin{document} 

\let\ref\cref\let\Ref\Cref

\title{Robust Classification of Graph-Based Data}

\author{Carlos~M.~Ala\'iz \and Micha\"el~Fanuel \and ~Johan~A.K.~Suykens}
\institute{C.M. Ala\'iz \at
            Universidad Aut\'onoma de Madrid - Departamento de Ingenier\'ia Inform\'atica \\
            Tom\'as y Valiente 11, 28049 Madrid - Spain \\
            \email{carlos.alaiz@uam.es}
           \and
           M. Fanuel and J.A.K. Suykens \at
            KU Leuven - Department of Electrical Engineering (ESAT--STADIUS) \\
            Kasteelpark Arenberg 10, B-3001 Leuven - Belgium
}

\maketitle

\begin{abstract}
A graph-based classification method is proposed for both semi-supervised learning in the case of Euclidean data and classification in the case of graph data.
Our manifold learning technique is based on a convex optimization problem involving a convex quadratic regularization term and a concave quadratic loss function with a trade-off parameter carefully chosen so that the objective function remains convex. As shown empirically, the advantage of considering a concave loss function is that the learning problem becomes more robust in the presence of noisy labels. Furthermore, the loss function considered here is then more similar to a classification loss while several other methods treat graph-based classification problems as regression problems.
\end{abstract}

\section{Introduction}
\label{SecIntroduction}

Nowadays there is an increasing interest in the study of graph-based data, either because the information is directly available as a network or a graph, or because the data points are assumed to be sampled on a low dimensional manifold whose structure is estimated by constructing a weighted graph with the data points as vertices. Moreover, fitting a function of the nodes of a graph, as a regression or a classification problem, can be a useful tool for example to cluster the nodes by using some partial knowledge about the partition and the structure of the graph itself.

In this paper, given some labelled data points and several other unlabelled ones, we consider the problem of predicting the label class of the latter.
Following the manifold learning framework, the data are supposed to be positioned on a manifold that is embedded in a high dimensional space, or to constitute a graph by themselves. In the first case, the usual assumption is that the classes are separated by low density regions, whereas in the second case is that the connectivity is weaker between classes than inside each of them~\citep{ChapelleBook}.
On the other side, the robustness of semi-supervised learning methods and their behaviour in the presence of noise, in this case just wrongly labelled data, has been recently discussed in~\cite{Gleich-2015-robustifying}, where a robustification method was introduced.
Notice that, out of the manifold learning framework, the literature regarding label noise is extensive, e.g. in~\cite{Natarajan2013} the loss functions of classification models are modified to deal with label noise, whereas in~\cite{Liu2016} a reweighting scheme is proposed. More recently, in~\cite{Vahdat2017} the label noise is modelled as part of a graphical model in a semi-supervised context.
There exists also a number of deep learning approaches for graph-based semi-supervised learning (see e.g.~\cite{Yang2016a}).

We propose here a different optimization problem, based on a concave error function, which is specially well-suited when the number of available labels is small and which can deal with flipped labels naturally.
The major contributions of our work are as follows:
\begin{enumerate}[(i)]
 \item We propose a manifold learning method phrased as an optimization problem which is robust to label noise. While many other graph-based methods involve a regression-like loss function, our loss function intuitively corresponds to a classification loss akin to the well-known hinge loss used in Support Vector Machines.
 \item We prove that, although the loss function is concave, the optimization problem remains convex provided that the positive trade-off parameter is smaller than the second least eigenvalue of the normalized combinatorial Laplacian of the graph.
 \item Computationally, the solution of the classification problem is simply given by solving a linear system, whose conditioning is described.
 \item In the case of Euclidean data, we present an out-of-sample extension of this method, which allows to extend the prediction to unseen data points.
 \item We present a heuristic method to automatically fix the unique hyper-parameter in order to get a parameter-free approach.
\end{enumerate}

Let us also emphasize that the method proposed in this paper can be naturally phrased in the framework of kernel methods, as a function estimation in a Reproducing Kernel Hilbert Space. Indeed, the corresponding kernel is then given by the Moore-Penrose pseudo-inverse of the normalized Laplacian. In this sense, this work can be seen as a continuation of~\cite{alaiz2016convex}.
    
The paper is structured as follows. \Ref{SecSemiClassification} introduces the context of the classification task and it reviews two state-of-the-art methods for solving it. In \ref{SecRobust} we introduce our proposed robust approach, which is numerically compared with the others in \ref{SecExperiments}. The paper ends with some conclusions in \ref{SecConclusions}.


\section{Classification of Graph-Based Data}
\label{SecSemiClassification}

\subsection{Preliminaries}

The datasets analysed in this paper constitute the nodes $\vers$ of a connected graph $\graph = \prn{\vers,\edgs}$, where the undirected edges $\edgs$ are given as a symmetric weight matrix $W$ with non-negative entries.
This graph can be obtained in different settings, e.g.:
\begin{itemize}
 \item Given a set of data points $\set{x_i}_{i = 1}^N$, with $x_i \in \R^d$ and a positive kernel $k\prn{x, y} \ge 0$, the graph weights can be defined as $w_{ij} = k\prn{x_i,x_j}$.
 \item Given a set of data points $\set{x_i}_{i = 1}^N$, with $x_i \in \R^d$, the weights are constructed as follows: $w_{ij} = 1$ if $j$ is among the $k$ nearest neighbours of $i$ for the $\lt$-norm, and $w_{ij} = 0$ otherwise. Then, the weight matrix $W$ is symmetrized as $\prn{W + W^\tr} / 2$.
 \item The dataset is already given as a weighted undirected graph.
\end{itemize}

Some nodes are labelled by $\pm 1$ and we denote by $\versl \subset \vers$ the set of labelled nodes. For simplicity, we identify $\versl$ with $\set{1,\dotsc, s}$ and $\vers$ with $\set{1,\dotsc, N}$, with $s < N$ the number of available labels. Any labelled node $i \in \versl$ has a class label $c_i = \pm 1$.
We denote by $y$ the label vector defined as follows
\begin{equation*}
 y_i = \begin{cases}
  c_i & \text{ if } i \in \versl , \\
  0   & \text{ if } i \in \vers \setminus \versl .
 \end{cases}
\end{equation*}

The methods discussed in this paper are formulated in the framework of manifold learning. Indeed, the classification of unlabelled data points relies on the definition of a Laplacian matrix, which can be seen as a discrete Laplace-Beltrami operator on a manifold~\citep{DiffusionMaps}.

Let $L = D - W$ be the matrix of the combinatorial Laplacian, where $D = \diag{d}$, $d$ is the degree vector $d = W 1$, and $1$ is the vector of ones, i.e., $d_i = \sum_{j = 1}^N w_{ij}$. We will write $i \sim j$ iff $w_{ij} \neq 0$.
The normalized Laplacian, defined by $\Ln = \Doh L \Doh = \iden - \Doh W \Doh$ (where $\iden \in \R^{N \times N}$ is the identity matrix), accounts for a non-trivial sampling distribution of the data points on the manifold. The normalized Laplacian has an orthonormal basis of eigenvectors $\set{\evec\ell}_{\ell = 0}^{N - 1}$, with $\evec{k}^\tr \evec\ell = \delta_{k \ell}$ (the Kronecker delta), associated to non-negative eigenvalues $0 = \eval0 \leq \eval1 \leq \dots \leq \eval{N-1} \leq 2$. Noticeably, the zero eigenvector of $\Ln$ is simply specified by the node degrees, i.e., we have $\evec{0,i} \propto \sqrt{d_i}$ for all $i = 1, \dotsc, N$. Notice that the Laplacian can be expressed in this basis according to the lemma below.
\begin{mylemma}
\label{Lem:spectral}
 The normalized Laplacian admits the following spectral decomposition, which also gives a resolution of the identity matrix $\iden \in \R^{N \times N}$:
 \begin{equation*}
  \Ln = \sum_{\ell = 1}^{N-1} \eval\ell \evec\ell \evec\ell^\tr \sepeq \iden = \sum_{\ell = 0}^{N-1} \evec\ell \evec\ell^\tr .
 \end{equation*}
\end{mylemma}
\begin{proof}
 See~\cite{ChungBook}.
\end{proof}
For simplicity, we assume here that each eigenvalue is associated to a one-dimensional eigenspace. The general case can be phrased in a straightforward manner.

Following~\cite{Belkin2004}, we introduce the smoothing functional associated to the normalized Laplacian:
\begin{equation}
\label{eq:Smoothness}
 S_\graph\prn{f} = \frac{1}{2} f^\tr \Ln f = \frac{1}{2} \sum_{i ,j | i \sim j} w_{ij} \prn{\frac{f_i}{\sqrt{d_i}} - \frac{f_j}{\sqrt{d_j}}}^2 ,
\end{equation}
where $f_i$ denotes the $i$-th component of $f$.

\begin{myremark}
 The smoothest vector according to the smoothness functional \ref{eq:Smoothness} is the eigenvector $\evec0$, which corresponds to a $0$ value, $S_\graph\prn{\evec0} = 0$.
\end{myremark}

The following sections contain a brief review of the state of the art for semi-supervised graph-based classification methods.

\subsection{Belkin--Niyogi Approach}
\label{SecBelkin}

In~\cite{Belkin2004}, a semi-supervised classification problem is phrased as the estimation of a (discrete) function written as a sum of the first $p$ smoothest functions, that is, the first $p$ eigenvectors of the combinatorial Laplacian.
The classification problem is defined by
\begin{equation}
\label[myprob]{eq:Belkin}
 \minp{a \in \R^p}{\sum_{i=1}^s \prn{c_i - \sum_{\ell = 0}^{p - 1} a_\ell \evec{\ell,i}}^2} ,
\end{equation}
where $a_0, \dotsc, a_{p - 1}$ are real coefficients.
The solution of \ref{eq:Belkin}, $a\opt$, is obtained by solving a linear system. The predicted vector is then
\begin{equation*}
 f\opt = \sum_{\ell = 1}^{p} a\opt_\ell \evec{\ell} .
\end{equation*}
Finally, the classification of an unlabelled node $i \in \vers \setminus \versl$ is given by $\sign{f\opt_i}$.
Indeed, \ref{eq:Belkin} is minimizing a sum of errors of a regression-like problem involving only the labelled data points. The information known about the position of the unlabelled data points is included in the eigenvectors $\evec\ell$ of the Laplacian (Fourier modes), which is the Laplacian of the full graph, including the unlabelled nodes. Only a small number $p$ of eigenvectors is used in order to approximate the label function. This number $p$ is a tuning parameter of the model.

We will denote this model as Belkin--Niyogi Graph Classification (\sbel{}).

\subsection{Zhou \emph{et al.} Approach}
\label{SecZhou}

In~\cite{Zhou}, the following regularized semi-supervised classification problem is proposed:
\begin{equation}
\label[myprob]{eq:Zhou}
 \minp{f \in \R^N}{\frac{1}{2} f^\tr \Ln f + \frac{\gamma}{2} \norm{f - y}^2} ,
\end{equation}
where $\gamma > 0$ is a regularization parameter which has to be selected.
We notice that the second term in the objective function of \ref{eq:Zhou}, involving the $\lt$-norm of the label vector, can be interpreted as the error term of a least-squares regression problem.
Intuitively, \ref{eq:Zhou} will have a solution $f\opt \in \R^N$ such that $f\opt_i \approx 0$ if $i \in \vers \setminus \versl$ (unlabelled nodes), that is, it will try to fit zeroes. Naturally, we will have $f\opt_i \approx c_i$ for all the labelled nodes $i \in \versl$.
Finally, the prediction of the unlabelled node class is given by calculating $\sign{f\opt_i}$ for $i \in \vers \setminus \versl$. The key ingredient is the regularization term (based on the Laplacian) which will make the solution smoother by increasing the bias.

Notice that the original algorithm solves \ref{eq:Zhou} once per each class, using as target the indicator vector of the nodes labelled as that class, and then classifying the unlabelled nodes according to the maximum prediction between all the classes. Although both formulations (using two binary target vectors and predicting with the maximum, or using a single target vector with $\pm 1$ and zero values and predicting with the sign) are slightly different, we will use \ref{eq:Zhou} since in this work we consider only binary problems.
We will denote this model as Zhou \emph{et al.} Graph Classification (\sreg{}).

In the recent work~\cite{Gleich-2015-robustifying}, it is emphasized that this method is implicitly robust in the presence of graph noise, since the prediction decays towards zero preventing the errors in far regions of the network from propagating to other areas. Moreover, a modification of this algorithm is proposed to add an additional $\lo$ penalization, so that the prediction decays faster according to an additional regularization parameter. However, the resultant method is still qualitatively similar to \sreg{} since the loss term is still the one of a regression problem, with the additional disadvantage of having an extra tuning parameter.

\subsection{Related Methods}

There are other semi-supervised learning methods that impose the label values as constraints~\citep{Zhu03semi-supervisedlearning,Joachims}. The main drawback is that, as discussed in~\cite{Gleich-2015-robustifying}, the rigid way of including the labelled information makes them more sensible to noise, specially in the case of mislabelled nodes.

On the other side, there are techniques with completely different approaches as Laplacian SVM~\citep{LapSVM}, a manifold learning model for semi-supervised learning based on an ordinary Support Vector Machine (SVM) classifier supplemented with an additional manifold regularization term. This method was originally designed for Euclidian data, hence its scope is different from the previous models. The straightforward approach to apply this method to graph data, is by embedding the graph, what in principle requires the computation of the inverse of a dense Gram matrix entering in the definition of an SVM problem. Hence, the training involves both a matrix inversion of the size of the labelled and unlabelled training data set and a quadratic problem of the same size. In order to reduce the computational cost, a training procedure in the primal was proposed in~\cite{LapSVMPrimal} where the use of a preconditioned conjugate gradient algorithm with an early stopping criterion is suggested. However, these methods still require the choice of two regularization parameters besides the kernel bandwidth. This selection requires a cross-validation procedure which is especially difficult if the number of known labels is small.

\section{Robust Method}
\label{SecRobust}

The two methods presented in \ref{SecBelkin,SecZhou} can be interpreted as regression problems, which intuitively estimate a smooth function $f\opt$ such that its value is approximately the class label, i.e., $f\opt_i \approx c_i$ for all the labelled nodes $i\in \versl$.
We will propose in this section a new method based on a concave loss function and a convex regularization term, which is best suited for classification tasks. Moreover, with the proper constraints, the resulting problem is convex and can be solved using a dual formulation.

We keep as a main ingredient the first term of \ref{eq:Zhou}, $\frac{1}{2} f^\tr \Ln f$, which is a well-known regularization term requiring a maximal smoothness of the solution on the (sampled) manifold.
However, if the smooth solution is $f\opt$, we emphasize that we have to favour  $\sign{f\opt_i} = c_i$ instead of imposing $f\opt_i \approx c_i$ for all $i \in \versl$. Hence, for $\gamma > 0$, we propose the minimization problem
\begin{equation}
\label[myprob]{eq:Robust1}
 \minpcl{f \in \R^N}{\frac{1}{2} f^\tr \Ln f - \frac{\gamma}{2} \sum_{i=1}^{N}{\prn{y_i + f_i}^2}}{f^\tr \evec0 = 0 ,}
\end{equation}
where $\gamma$ has to be bounded from above as stated in \ref{Thm:Main}.
The constraint means that we do not want the solution to have a component directed along the vector $\evec0$, since its components all have the same sign (an additional justification is given in \ref{Rem:RKHS}).
We will denote our model as Robust Graph Classification (\srob{}).

Notice that \ref{eq:Robust1}, corresponding to \srob{}, can be written as \ref{eq:Zhou}, corresponding to \sreg{}, by doing the following changes: $\gamma \to -\gamma$, $y \to -y$, and by supplementing the problem with the constraint $f^T v_0 = 0$.
Both problems can be compared by analysing the error term in both formulations.
In \sreg{} this term simply corresponds to the Squared Error (SE), namely $\prn{f_i - y_i}^2$.
In \srob{}, a Concave Error (CE) is used instead, $- \prn{f_i + y_i}^2$.
As illustrated in \ref{FigLossFuncions}, this means that \sreg{} tries to fit the target, both if it is a known label $\pm 1$, or if it is zero. On the other side, \srob{} tries to have predictions far from $0$ (somehow minimizing the entropy of the prediction), biased towards the direction marked by the label for labelled points. Nevertheless, as shown in \ref{FigLossFuncionsP}, the model is also able to minimize the CE in the opposite direction to the one indicated by the label, what provides robustness with respect to label noise. Finally, if the label is unknown, the CE only favours large predictions in absolute value.
As an additional remark, let us stress that the interplay of the Laplacian-based regularization and the error term, which are both quadratic functions, is yet of fundamental importance. As a matter of fact, in the absence of the regularization term, the minimization of the unbounded error term is meaningless.

{
 \renewcommand{\showlegend}[1]{[\,\raisebox{1.5pt}{\plotline{LossFunctions}{#1}}\,]}
 \begin{mfigure}{\label{FigLossFuncions} Comparison of the Squared Error (SE) and the proposed Concave Error (CE), both for a labelled node with $c_i = 1$ (the case $c_i = -1$ is just a reflection of this one) and for an unlabelled point. \\ Legend: \showlegend{0}~SE; \showlegend{1}~CE.}
  \subfloat[\label{FigLossFuncionsP} Positive label.]{\includeg{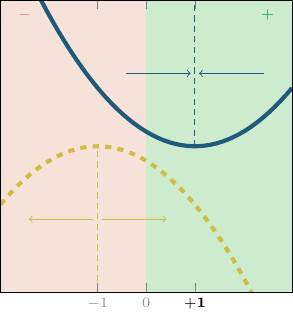}}\quad%
  \subfloat[\label{FigLossFuncionsU} Unknown label.]{\includeg{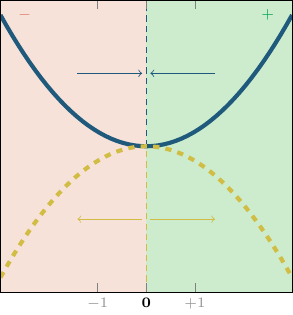}}
 \end{mfigure}
}

\srob{} can be further studied by splitting the error term to get the following equivalent problem:
\begin{equation*}
\label[myprob]{eq:Robust2}
 \minpcl{f \in \R^N}{\frac{1}{2} f^\tr \Ln f + \gamma \sum_{i=1}^{N}{\prn{-y_i f_i}} + \gamma \sum_{i=1}^{N}{\prn{- \frac{f_i^2}{2}}}}{f^\tr \evec0 = 0 ,}
\end{equation*}
where the two error terms have the following meaning:
\begin{itemize}
\item The first error term is a penalization term involving a sum of loss functions $\loss{f_i} = - y_i f_i$. This unbounded loss function term is reminiscent of the hinge loss in Support Vector Machines: $\maxp{0, 1 - y_i f_i}$. Indeed, for each labelled node $i \in \versl$, this term favours values of $f_i$ which have the sign of $y_i$. However, for each unlabelled node $i \in \vers \setminus\versl$, the corresponding term $\loss{f_i} = 0$ vanishes. This motivates the presence of the other error term.
\item The second error term is a penalization term forcing the value $f_i$ to take a non-zero value in order to minimize $- f_i^2 / 2$. In particular, if $i$ is unlabelled, this term favours $f_i$ to take a non-zero value which will be dictated by the neighbours of $i$ in the graph.
\end{itemize}

The connection between our method and kernel methods based on a function estimation problem in a Reproducing Kernel Hilbert Space (RKHS) is explained in the following remark.
\begin{myremark}
\label{Rem:RKHS}
 The additional condition $f^\tr \evec0  = 0$ in \ref{eq:Robust1} can also be justified as follows. The Hilbert space $H_K = \set{f \in \R^N \text{ s.t. } f^\tr \evec0 = 0}$ is an RKHS endowed with the inner product $\inp{f}{f'}_K = f^\tr \Ln f'$ and with the reproducing kernel given by the Moore--Penrose pseudo-inverse $K = \prn{\Ln}^\dagger$. 
 More explicitly, we can define $K_i = \prn{\Ln}^\dagger e_i \in \R^N$, where $e_i$ is the canonical basis element given by a vector of zeros with a $1$ at the $i$-th component. Furthermore, the kernel evaluated at any nodes $i$ and $j$ is given by $K\prn{i,j} = e_i^\tr \prn{\Ln}^\dagger e_j$. As a consequence, the reproducing property is merely~\citep{zhou2011iterated}
 \begin{equation*}
  \inp{K_i}{f}_K = \prn{\prn{\Ln}^\dagger e_i}^\tr \Ln f = f_i ,
 \end{equation*}
 for any $f\in H_K$.
 As a result, the first term of \ref{eq:Robust1} is equal to $\normp{f}{K}^2 / 2$ and the problem becomes a function estimation problem in an RKHS.
\end{myremark} 

Notice that the objective function involves the difference of two convex functions and, therefore, it is not always bounded from below. The following theorem states the values of the regularization parameter such that the objective is bounded from below on the feasible set and so that the optimization problem is convex.
\begin{mythm}
\label{Thm:Main}
 Let $\gamma > 0$ be a regularization parameter. The optimization problem
 \begin{equation*}
  \minpc{f \in \R^N}{\frac{1}{2} f^\tr \Ln f - \frac{\gamma}{2} \norm{f+y}^2}{f^\tr \evec0 = 0}
 \end{equation*}
 has a strongly convex objective function on the feasible space if and only if $\gamma < \eval1$, where $\eval1$ is the second smallest eigenvalue of $\Ln$. In that case, the unique solution is given by the vector:
 \begin{equation*}
  f\opt = \prn{\frac{\Ln}{\gamma} - \iden}^{-1} \pz{y} ,
 \end{equation*}
 with $\pz{} = \iden - \evec0 \evec0^\tr$.
\end{mythm}
\begin{proof}
 Using \ref{Lem:spectral}, any vector $f \in \evec0^\perp$, i.e., satisfying the constraint $f^\tr \evec0 = 0$, can be written as $f = \sum_{\ell = 1}^{N-1} \tilde{f}_\ell \evec\ell$, where $\tilde{f}_\ell = \evec\ell^\tr f \in \R$ is the projection of $f$ over $\evec\ell$. Furthermore, we also expand the label vector in the basis of eigenvectors $y = \sum_{\ell = 0}^{N-1} \tilde{y}_\ell \evec\ell$, with $\tilde{y}_\ell = \evec\ell^\tr y$. Then, the objective function is the finite sum
\begin{equation*}
 F\prn{\tilde{f}_1, \dotsc, \tilde{f}_{N-1}} = \sum_{\ell = 1}^{N-1}\prn{\frac{\eval\ell - \gamma}{2} \tilde{f}^2_\ell - \gamma \tilde{y}_\ell \tilde{f}_\ell} - \frac{\gamma}{2} \norm{y}^2 ,
\end{equation*}
 where we emphasize that the term $\ell = 0$ is missing. As a result, $F$ is clearly a strongly convex function of $\prn{\tilde{f}_1, \dotsc, \tilde{f}_{N-1}}$ if and only if $\gamma < \eval\ell$ for all $\ell = 1, \dotsc, N-1$, that is, iff $\gamma < \eval1$.
 Since the objective $F$ is quadratic, its minimum is merely given by $\prn{\tilde{f}\opt_1, \dotsc, \tilde{f}\opt_{N-1}}$, with
 \begin{equation}
 \label{eq:Proof}
  \tilde{f}\opt_\ell = \frac{\tilde{y}_\ell}{\frac{\eval\ell}{\gamma} - 1} , 
 \end{equation}
 for $\ell = 1, \dotsc, N-1$. Then, the solution of the minimization problem is given by
 \begin{align*}
  f\opt = \sum_{\ell = 1}^{N-1} \tilde{f}\opt_\ell \evec\ell &= \sum_{\ell = 1}^{N-1} \frac{\tilde{y}_\ell}{\frac{\eval\ell}{\gamma} - 1} \evec\ell \\
   &= \prn{\frac{\Ln}{\gamma} - \iden}^{-1} \prn{y - \evec0 \prn{\evec0^\tr y}} ,
 \end{align*}
 which is obtained by using the identity $y - \evec0 \prn{\evec0^\tr y} = \sum_{\ell = 1}^{N-1} \tilde{y}_\ell \evec\ell$. This completes the proof.
\end{proof}

By examining the form of the solution of \ref{eq:Robust1} given in \ref{eq:Proof} as a function of the regularization constant $0 < \gamma < \eval1$, we see that taking $\gamma$ close to the second eigenvalue $\eval1$ will give more weight to the second eigenvector, while the importance of the next eigenvectors decreases as $1 / \eval\ell$.
Regarding the selection of $\gamma$ in practice, as shown experimentally just fixing a value of $\gamma = \num{0.9} \eval1$ leads to a parameter-free version of \srob{} (denoted \srobf{}) that keeps a considerable accuracy.

The complete procedure to apply this robust approach is summarized in \ref{AlgRobust}, where $\gamma$ is set as a percentage $\eta$ of $\eval1$ to make it problem independent.
Notice that, apart from building the needed matrices and vectors, the algorithm only requires to compute the largest eigenvalue of a matrix and to solve a well-posed linear system.

\begin{malgorithm}{\label{AlgRobust} Algorithm of \srob{}.}
 \begin{algorithmic}[1]
  \REQUIRE {\quad} \\
   \algreq Graph $\graph$ given by the weight matrix $W$ \algstop \\
   \algreq Regularization parameter $0 < \eta < 1$ \algstop \\
  \ENSURE {\quad} \\
   \algreq Predicted labels $\hat{y}$ \algstop
  \STATE $d_{ii} \gets \sum_j {W_{ij}}$ \algstop
  \STATE $S \gets \Doh W \Doh$ \algstop
  \STATE $\Ln \gets \iden - S$ \algstop
  \STATE $(\evec0)_i \gets \sqrt{d_{ii}}$ \algstop
  \STATE $\evec0 \gets \evec0 / \|\evec0\|$ \algstop
  \STATE Compute $\eval1$, second smallest eigenvalue of $\Ln$, or, alternatively, largest eigenvalue of $S - \evec0 \evec0^\tr$ \algstop
  \STATE $\gamma \gets \eta \eval1$ \algstop
  \STATE $f \gets \prn{\Ln / \gamma -\iden}^{-1} \prn{y - \evec0 \prn{\evec0^\tr y}}$ \algstop
  \RETURN $\hat{y} \gets \sign{f}$ \algstop
 \end{algorithmic}
\end{malgorithm}

\subsection{Illustrative Example}

A comparison of \sreg{}, \sbel{} and \srob{} is shown in \ref{FigIllustrative}, where the three methods are applied over a very simple graph: a chain with strong links between the first ten nodes, strong links between the last ten nodes, and a weak link connecting the tenth and the eleventh nodes (with a weight ten times smaller). This structure clearly suggests to split the graph in two halves.

{
 \renewcommand{\showlegend}[1]{$\left [ \begin{minipage}{22pt}\centering \plotline{Illustrative}{#1} \\ \plotpattern{Illustrative}{#1} \end{minipage} \right ]$}
 \begin{mfigure}{\label{FigIllustrative} Comparison of the different methods over a chain with two clearly separable clusters, where the link between the two middle nodes is ten times smaller than the other links. \\ Legend: \showlegend{0}~\sreg{}; \showlegend{1}~\sbel{}; \showlegend{2}~\srob{}.}
  \subfloat[\label{FigIllustrativeA} Example with two correct labels.]{\includeg{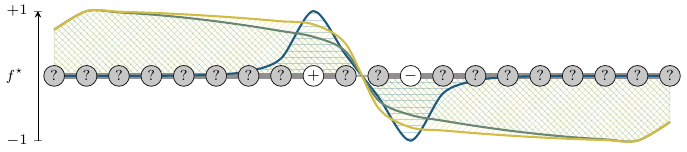}}\\
  \subfloat[\label{FigIllustrativeB} Example with four correct labels and a flipped one.]{\includeg{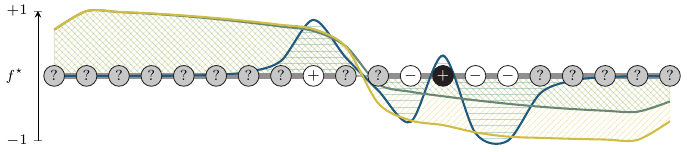}}
 \end{mfigure}
}

In \ref{FigIllustrativeA} one node of each cluster receives a label, whereas in \ref{FigIllustrativeB} one node of the positive class and four of the negative are labelled, with a flipped label in the negative class.
The predicted values of $f\opt$ show that \sreg{} (with $\gamma = \num{1}$) is truly a regression model, fitting the known labels (even the flipped one) and pushing towards zero the unknown ones.
\sbel{} (with two eigenvectors, $p = \num{2}$) fits much better the unknown labels for nodes far from the labelled ones, although the flipped label push the prediction towards zero in the second example for the negative class.
Finally, \srob{} (with $\eta = \num{0.5}$) clearly splits the graph in two for the first example, where the prediction is almost a step function, and it is only slightly affected by the flipped label of the second example.
Of course, this experiment is only illustrative, since tuning the parameters of the different models could affect significantly the results.

\subsection{Conditioning of the Linear System}
\label{SecConditioning}

As shown in \ref{Thm:Main}, the \srob{} model is trained by solving the following linear system:
\begin{equation*}
 \prn{\frac{\Ln}{\gamma} - \iden} f\opt = \pz{y} .
\end{equation*}
It is therefore interesting to describe the condition number of this system in order to estimate the stability of its numerical solution. In particular, we will use the following lemma characterizing the maximum eigenvalue of $\Ln$.
\begin{mylemma}
If the weight matrix is positive semi-definite, $W \succeq 0$, then $\evmax{\Ln} \leq 1$. If $W$ is indefinite, then $\evmax{\Ln} \leq 2$.
\end{mylemma}
\begin{proof}
The argument is classic.
Let us write $\Ln = \iden - S$, with $S = \Doh W \Doh$. Clearly, $S$ is related by the conjugation to a stochastic matrix $\Sigma = D^{-1} W = \Doh S D^{1/2}$. Hence, $\Sigma$ and $S$ have the same spectrum $\set{\eval\ell}_{\ell = 0}^{N-1}$. Therefore, since $\Sigma$ is stochastic, it holds that $\abs{\eval\ell} \leq 1$ for all $\ell = 0, \dotsc, N-1$. Then, in general, $\evmax{\Ln} = 1 - \evmin{S} \le 2$, which proves the second part of the Lemma. Furthermore, if $W \succeq 0$, then $S \succeq 0$, which means that $\evmin{S} \geq 0$ and we have $\evmax{\Ln} = 1 - \evmin{S} \leq 1$, which shows the first part of the statement.
\end{proof}

Furthermore, in the feasible space (i.e., for all $f \in \R^N$ such that $f^\tr \evec0 = 0$), we have $\evmin{\Ln} = \eval1$.
Then, we can deduce the condition number of the system:
\begin{equation*}
 \kappa = \frac{\abs{\evmax{\Ln / \gamma - \iden}}}{\abs{\evmin{\Ln / \gamma - \iden}}} \leq \frac{c - \gamma}{\eval1 - \gamma} ,
\end{equation*}
where $c = 1$ if the weight matrix is positive semi-definite and $c = 2$ if the weight matrix is indefinite.

The upshot is that the problem is better conditioned a priori if the weight matrix is positive semi-definite. Furthermore, in order to have a reasonable condition number, $\gamma$ should not be too close to $\eval1$.

\subsection{Out-of-Sample Extension}

In the particular case of a graph obtained using a Mercer kernel over a set of data points $\set{x_i}_{i = 1}^N$, with $x_i \in \R^d$, an out-of-sample extension allows to make predictions over unseen points.

In order to pose an out-of-sample problem, let $f\opt = \prn{\Ln / \gamma - \iden}^{-1} \pz{y}$ be the solution of the \srob{} problem. Then, if we are given an additional point $x \in \R^d$, we want to obtain the value of the classifier $f_x$ such that $\sign{f_x}$ predicts the label of $x$.
In particular, recall that the normalized Laplacian $\Ln = \iden - S$ is built from the kernel matrix
\begin{equation*}
 S_{ij} = \frac{k\prn{x_i, x_j}}{\sqrt{d_i d_j}} , \text{ with } d_i = \sum_{j=1}^N k\prn{x_i, x_j} .
\end{equation*}
This kernel can be extended to the new point $x$ as follows:
\begin{equation*}
 S_{xj} = \frac{k\prn{x, x_j}}{\sqrt{d_x d_j}} , \text{ with } d_x = \sum_{j=1}^N k\prn{x, x_j} ,
\end{equation*}
and $S_{xx} = k\prn{x, x} / d_x$ (notice that in many of the most common kernels, such as the Gaussian kernel, $k\prn{x, x} = 1$).
We consider
\begin{equation*}
 \tilde{f} = \begin{pmatrix}
  f\opt \\
  f_x
 \end{pmatrix}
 \text{ and }
 \tilde{y} = \begin{pmatrix}
  y \\
  0
 \end{pmatrix} ,
\end{equation*}
with $\tilde{f}, \tilde{y} \in \R^{N+1}$ and $f, y \in \R^N$. The extension of the Laplacian is defined as follows:
\begin{equation*}
 \Lnt = \begin{pmatrix}
  \Ln & l \\
  l^\tr & \Lnt_{xx}
 \end{pmatrix},
\end{equation*}
with $l_i = - S_{xi}$ and $\Lnt_{xx} = 1 - k\prn{x, x} / d_x$. Notice that $\Lnt$ is not necessarily positive semi-definite.

In order to obtain $f_x$, we propose the minimization problem
\begin{equation*}
 \minpc{\tilde{f} \in \R^{N+1}}{\frac{1}{2} \tilde{f}^\tr \Lnt \tilde{f} -\frac{\gamma}{2} \sum_{i=1}^{N}{\prn{\tilde{y}_i + \tilde{f}_i}^2}}{\tilde{f} = \begin{pmatrix}
  f\opt \\
  f_x
 \end{pmatrix}} ,
\end{equation*}
which is equivalent to solving
\begin{equation*}
 \minp{f_x \in \R}{\frac{\Lnt_{xx} - \gamma}{2} f_x^2 + \prn{l^\tr f\opt} f_x} ,
\end{equation*}
where $l^\tr f\opt = - \sum_{i=1}^N S_{xi} f\opt_i$.
This quadratic problem has a solution provided that $\Lnt_{xx} - \gamma > 0$, that is, only if the degree of this new point $x$ is large enough:
\begin{equation*}
 d_x > \frac{k\prn{x, x}}{1 - \gamma} .
\end{equation*}
This means that $x$ has to be close enough from the initial set $\set{x_i}_{i = 1}^N$ in order to be able to extend the classifier given by $f\opt$ (notice that, in this case, $\gamma < \eval1 < 1$, and hence the inequality involving $d_x$ is well defined).
Under this assumption, the solution reads
\begin{equation}
 \label{eq:oosextension}
 f_x = \frac{1}{1 - \gamma - k\prn{x, x} d_x^{-1}} \sum_{i=1}^N S_{xi} f\opt_i ,
\end{equation}
namely, it is a Nystr\"om-like extension of the solution $f\opt$ with respect to the Mercer kernel $S\prn{x, y} = k\prn{x, y} / \sqrt{d_x d_y}$.

\subsubsection*{Example of the Out-of-Sample Extension}

\Ref{FigOOSExtension} includes an example of the out-of-sample extension over the \texttt{moons} dataset, with \num{50} patterns (\num{5} of them labelled) for each class. The model built is then extended over a $\num{100} \times \num{100}$ regular grid using \ref{eq:oosextension}. As the bandwidth of the kernel is extended, the prediction can be extended to a broader area, but the classification becomes less precise at the border between the two classes.

{
 \renewcommand{\showlegend}[1]{[\,\raisebox{1.5pt}{\plotline{OOSExtension}{#1}}\,]}
 \newcommand{\showlegendaux}[1]{[\plotpattern{OOSExtension}{#1}]}
 \begin{mfigure}{\label{FigOOSExtension} Out-of-sample extension over the \texttt{moons} dataset using different bandwidths. \\ Legend: \showlegend{0}/\showlegend{1}~unlabelled point of class $-1$/$+1$; \showlegend{2}/\showlegend{3}~labelled point of class $-1$/$+1$; \showlegendaux{-1}/\showlegendaux{1}~area predicted as class $-1$/$+1$; \showlegendaux{0}~area out of prediction range.}
  \subfloat[Bandwidth $\sigma = \num{0.15}$.]{\includeg{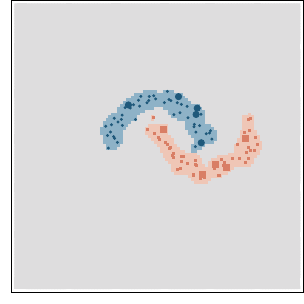}}\quad%
  \subfloat[Bandwidth $\sigma = \num{0.3}$.]{\includeg{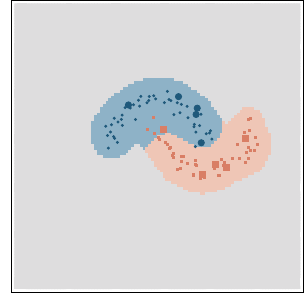}}\\
  \subfloat[Bandwidth $\sigma = \num{0.6}$.]{\includeg{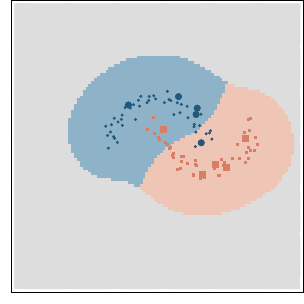}}\quad%
  \subfloat[Bandwidth $\sigma = \num{1.2}$.]{\includeg{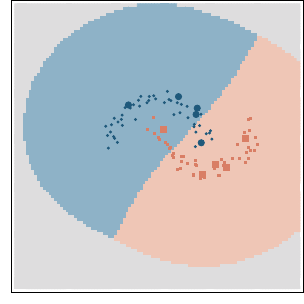}}
 \end{mfigure}
}

\section{Experiments}
\label{SecExperiments}

In this section we will illustrate the robustness of the proposed method \srob{} with respect to labelling noise, we will show empirically how it can be successfully applied to the problem of classifying nodes over different graphs, and we will also include an example of its out-of-sample extension.

For the first two set of experiments, the following four models will be compared:
\begin{description}
 \item[\sreg{}] It corresponds to \ref{eq:Zhou}, where the parameter $\gamma$ is selected from a grid of \num{51} points in logarithmic scale in the interval $\brq{\num{e-5}, \num{e5}}$.
 \item[\sbel{}] It corresponds to \ref{eq:Belkin}. The number $p$ of eigenvectors used is chosen between \num{1} and \num{51}.
 \item[\srob{}] It corresponds to \ref{eq:Robust1}, where the parameter $\gamma$ is selected from a grid of \num{51} points in linear scale between $0$ and $\eval1$.
 \item[\srobf{}] It corresponds to \ref{eq:Robust1}, where $\gamma$ is fixed as $\gamma = \num{0.9} \eval1$, so it is a parameter-free method. As shown in \ref{FigTuning}, the stability of the prediction with respect to $\gamma$ suggests to use such a fixed value, where the solution is mainly dominated by the second eigenvector $\evec1$ but without ignoring the next eigenvectors. Moreover, a value of $\gamma$ closer to $\eval1$ could affect the conditioning of the linear system, as explained in \ref{SecConditioning}.
\end{description}

Regarding the selection of the tuning parameters, these models are divided in two groups:
\begin{itemize}
 \item For \sreg{}, \sbel{} and \srob{}, a perfect validation criterion is assumed, so that the best parameter is selected according to the test error. Although this approach prevents from estimating the true generalization error, it is applied to the three models so that the comparison between them should still be fair, and this way we avoid the crucial selection of the parameter, which can be particularly difficult for the small sizes of labelled set considered here. Obviously, any validation procedure will give results at best as good as these ones.
 \item \srobf{} does not require to set any tuning parameter, hence its results are more realistic than those of the previous group, and it is in disadvantage with respect to them. This means that, if this model outperforms the others in the experiments, it is expected to do it in a real context, where the parameters of the previous methods have to be set without using test information.
\end{itemize}

\subsection{Robustness of the Classification with respect to Label Noise}

The first set of experiments aims to test the robustness of the classification of the different models with respect to label noise. In particular, we propose to generate a Stochastic Block Model as follows: a very simple graph of \num{200} nodes with two clusters is generated with an intra-cluster connectivity of \num{70}\%, whereas the connectivity between clusters is either \num{30}\% (a well-separated problem) or \num{50}\% (a more difficult problem); an example of the resulting weight matrices is shown in \ref{FigConnectivity}.
For each of these two datasets, the performance of the models is compared for different numbers of labels and different levels of noise, which correspond to the percentage of flipped labels. Each configuration is repeated \num{50} times by varying the labelled nodes to average the accuracies.

 \begin{mfigure}{\label{FigConnectivity} Binary weight matrices for the Stochastic Block Model with low and high inter-cluster connectivity (the connections are marked in yellow).}
  \subfloat[Low inter-cluster connectivity.]{\includeg{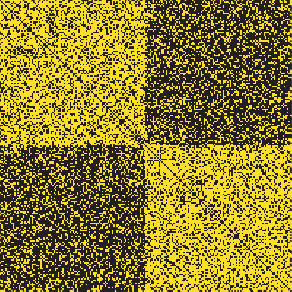}}\quad%
  \subfloat[High inter-cluster connectivity.]{\includeg{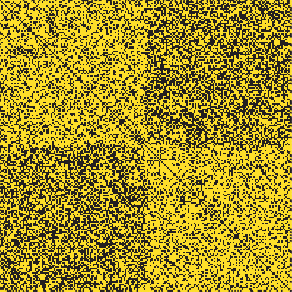}}
 \end{mfigure}

The results are included in \ref{FigRobust}, where the solid lines represent the average accuracy, and the striped regions the areas between the minimum and maximum accuracies.
In the case of the low inter-cluster connectivity dataset (left column of \ref{FigRobust}), \srob{} is able to almost perfectly classify all the points independently of the noise level (hence the striped region only appears when the number of labels is small and the noise is maximum). Moreover, \srobf{} is almost as good as \srob{}, and only slightly worse when the noise is the highest and the number of labels is small. These two models outperform \sbel{}, and also \sreg{}, which is clearly the worse of the four approaches.
Regarding the high inter-cluster connectivity dataset (right column of \ref{FigRobust}), for this more difficult problem \srob{} still gets a perfect classification except when the noise level is very high, where the accuracy drops a little when the number of labels is small. \sbel{} is again worse than \srob{}, and the difference is more noticeable when the noise increases. On the other side, the heuristic \srobf{} is in this case worse than \sbel{} (the selection of $\gamma$ is clearly not optimal) but it still outperforms \sreg{}.

{
 \renewcommand{\showlegend}[1]{$\left [ \begin{minipage}{22pt}\centering \plotline{Robust}{#1} \\ \plotpattern{Robust}{#1} \end{minipage} \right ]$}
 \begin{mfigure}{\label{FigRobust} Robust comparison for the low inter-cluster connectivity graph (left column) and the high inter-cluster connectivity graph (right column). \\ Legend: \showlegend{0}~\sreg{}; \showlegend{1}~\sbel{}; \showlegend{2}~\srob{}; \showlegend{3}~\srobf{}.}
   \begin{tabular}{@{}l@{}r@{}}
    \multicolumn{1}{r}{\makebox[0.405\textwidth][c]{\textbf{Low Connectivity}}} & \multicolumn{1}{l}{\makebox[0.405\textwidth][c]{\textbf{High Connectivity}}} \\[1pt]
    {\includeg{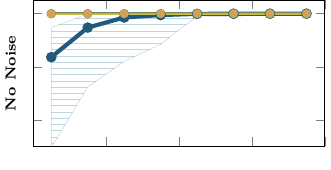}} & \includeg{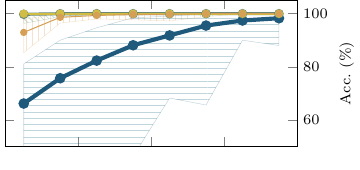}\\
    {\includeg{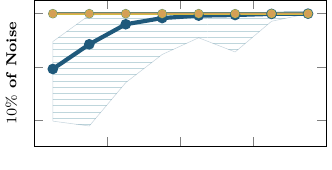}} & \includeg{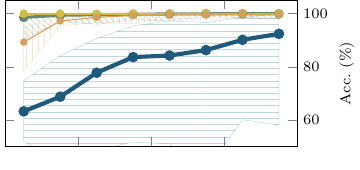}\\
    {\includeg{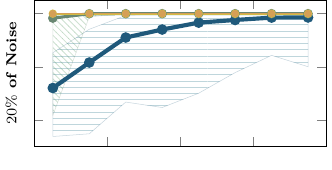}} & \includeg{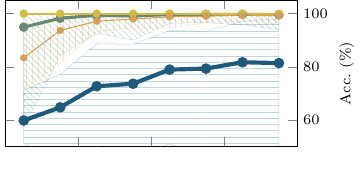}\\
    {\includeg{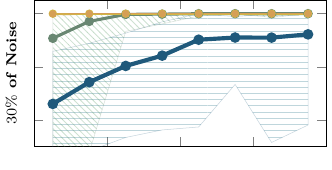}} & \includeg{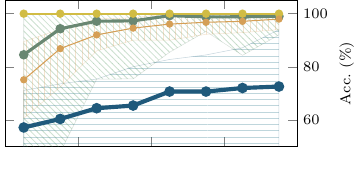}\\
    {\includeg{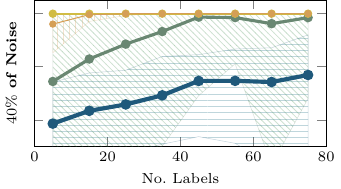}} & \includeg{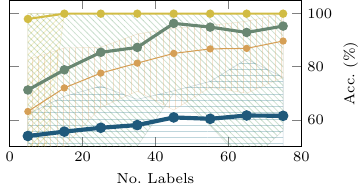}
   \end{tabular}
 \end{mfigure}
}

\subsection{Accuracy of the Classification}

The second set of experiments consists in predicting the label of the nodes over the following six supervised datasets:
\begin{description}
 \item[\texttt{digits49-s} \textnormal{and} \texttt{digits49-w}] The task is to distinguish between the handwritten digits $4$ and $9$ from the USPS dataset~\citep{hull1994database}; the suffix \texttt{-s} denotes that the weight matrix is binary and sparse corresponding to the symmetrized \num{20}-Nearest Neighbours graph, whereas the suffix \texttt{-w} corresponds to a non-sparse weight matrix built upon a Gaussian kernel with $\sigma = \num{1.25}$. The total number of nodes is \num{250} (\num{125} of each class).
 \item[\texttt{karate}] This dataset corresponds to a social network of \num{34} people of a karate club, with two communities of sizes \num{18} and \num{16}~\citep{zachary1977information}.
 \item[\texttt{polblogs}] A symmetrized network of hyperlinks between weblogs on US politics from \num{2005}~\citep{adamic2005political}; there are \num{1222} nodes, with two clusters of \num{636} and \num{586} elements.
 \item[\texttt{polbooks}] A network of books about US politics around \num{2004} presidential election, with \num{92} nodes and two classes of \num{49} and \num{43} elements.
 \item[\texttt{synth}] This  dataset is generated by a Stochastic Block Model composed by three clusters of $100$ points with a connectivity of $30\%$ inside each cluster and $5\%$ between clusters; the positive class is composed by one cluster and the negative by the other two.
\end{description}
For each dataset, \num{6} different training sizes (or number of labelled nodes) are considered, corresponding to \num{1}\%, \num{2}\%, \num{5}\%, \num{10}\%, \num{20}\% and \num{50}\% the total number of nodes, provided that this number is larger than two, since at least one sample of each class is randomly selected.
Moreover, each experiment is repeated \num{20} times by varying the labelled nodes in order to average the result and check if the differences between them are significant.
In order to compare the models we use the accuracy over the unlabelled samples.

The results are included in \ref{TabResults}, where the first column includes the percentage of labelled data and the corresponding number of labels, and the other four columns show the accuracy (mean and standard deviation) of the four models and the corresponding ranking (given also by the the colour; the darker, the better). The same rank value is repeated if two models are not significantly different\footnote{Using a Wilcoxon signed rank test for zero median, with a significance level of \num{5}\%.}.
We can see that the proposed \srob{} method outperforms both \sreg{} and \sbel{} at least for the smallest training sizes, and for all the sizes in the cases of \texttt{karate}, \texttt{polblogs} (the largest one) and \texttt{polbooks}.
In the case of \texttt{digits49-s} and \texttt{digits49-w}, \srob{} beats the other methods for the three first sizes, being then beaten by \sbel{} in the former and \sreg{} in the latter. Finally, for \texttt{synth} the robust \srob{} is the best model for the smallest training size, but it is then outperformed by \sbel{} until the largest training size, where both of them solve the problem perfectly. Notice that this dataset is fairly simple, and a spectral clustering approach over the graph (without any labels) could be near a correct partition; \sbel{} can benefit for this partition just regressing over the first eigenvectors to get a perfect classification with a very small number of labels.
Turning our attention to the parameter-free heuristic approach \srobf{}, it is comparable to the approach with perfect parameter selection \srob{} in \num{3} out of the \num{6} datasets. In \texttt{digits49-s}, \texttt{digits49-w} and \texttt{synth}, \srobf{} is comparable to \srob{} for the experiments with a small number of labels, although it works slightly worse when the number of labels is increased. Nevertheless, the results show that the proposed heuristic performs quite well in practice.

\begin{mtable}{\label{TabResults}Accuracy of the classification.}
\setlength{\tabcolsep}{0pt}
 \begin{tabular}{c@{\quad}l@{\quad}*4{c@{\,}}}
  \toprule
   \rotatebox[origin=c]{90}{\tabformathead{Data}} & \tabformathead{Labs.} & \tabformathead{\sreg{}} & \tabformathead{\sbel{}} & \tabformathead{\srob{}} & \tabformathead{\srobf{}} \\
  \midrule
\multirow{6}{*}{\rotatebox[origin=c]{90}{\texttt{digits49-s}}}
 & $\enspace 1\%$ (2) & \formatd{$\num{76.6}\pm\num{14.7}$} & \formatd{$\num{74.5}\pm\num{19.6}$} & \formata{$\num{79.1}\pm\num{16.4}$} & \formata{$\num{77.4}\pm\num{20.1}$}\\
 & $\enspace 2\%$ (5) & \formatd{$\num{80.1}\pm\num{9.4}$} & \formatb{$\num{81.6}\pm\num{11.5}$} & \formata{$\num{86.9}\pm\num{4.7}$} & \formatb{$\num{85.7}\pm\num{1.9}$}\\
 & $\enspace 5\%$ (12) & \formatd{$\num{85.8}\pm\num{4.0}$} & \formata{$\num{88.2}\pm\num{2.4}$} & \formata{$\num{88.7}\pm\num{2.7}$} & \formatd{$\num{85.0}\pm\num{1.6}$}\\
 & $10\%$ (25) & \formatb{$\num{89.3}\pm\num{2.6}$} & \formata{$\num{91.1}\pm\num{4.5}$} & \formatb{$\num{89.2}\pm\num{2.2}$} & \formatd{$\num{85.0}\pm\num{1.0}$}\\
 & $20\%$ (50) & \formatb{$\num{92.3}\pm\num{2.4}$} & \formata{$\num{94.5}\pm\num{2.6}$} & \formatc{$\num{89.7}\pm\num{1.9}$} & \formatd{$\num{84.8}\pm\num{1.4}$}\\
 & $50\%$ (125) & \formatb{$\num{94.8}\pm\num{2.0}$} & \formata{$\num{98.1}\pm\num{1.0}$} & \formatc{$\num{90.1}\pm\num{1.9}$} & \formatd{$\num{84.5}\pm\num{2.0}$}\\
\midrule
\multirow{6}{*}{\rotatebox[origin=c]{90}{\texttt{digits49-w}}}
 & $\enspace 1\%$ (2) & \formatd{$\num{70.1}\pm\num{13.4}$} & \formata{$\num{74.4}\pm\num{9.9}$} & \formata{$\num{75.5}\pm\num{13.3}$} & \formata{$\num{75.1}\pm\num{14.0}$}\\
 & $\enspace 2\%$ (5) & \formata{$\num{81.6}\pm\num{9.8}$} & \formatd{$\num{70.6}\pm\num{15.7}$} & \formata{$\num{82.7}\pm\num{7.4}$} & \formata{$\num{81.4}\pm\num{8.7}$}\\
 & $\enspace 5\%$ (12) & \formata{$\num{87.9}\pm\num{4.7}$} & \formata{$\num{85.4}\pm\num{9.4}$} & \formata{$\num{85.5}\pm\num{5.1}$} & \formatd{$\num{84.4}\pm\num{5.2}$}\\
 & $10\%$ (25) & \formata{$\num{93.9}\pm\num{2.1}$} & \formatd{$\num{89.3}\pm\num{5.9}$} & \formatd{$\num{90.1}\pm\num{4.7}$} & \formatd{$\num{89.1}\pm\num{4.0}$}\\
 & $20\%$ (50) & \formata{$\num{95.7}\pm\num{1.3}$} & \formatb{$\num{91.9}\pm\num{2.6}$} & \formatb{$\num{92.5}\pm\num{2.8}$} & \formatd{$\num{89.7}\pm\num{3.5}$}\\
 & $50\%$ (125) & \formata{$\num{96.9}\pm\num{1.3}$} & \formatb{$\num{95.4}\pm\num{1.7}$} & \formatb{$\num{94.5}\pm\num{2.6}$} & \formatd{$\num{89.6}\pm\num{2.9}$}\\
\midrule
\multirow{6}{*}{\rotatebox[origin=c]{90}{\texttt{karate}}}
 & \emptyrow
 & \emptyrow
 & $\enspace 5\%$ (2) & \formatd{$\num{90.3}\pm\num{12.2}$} & \formatc{$\num{95.5}\pm\num{7.3}$} & \formata{$\num{98.9}\pm\num{1.5}$} & \formata{$\num{98.9}\pm\num{1.5}$}\\
 & $10\%$ (3) & \formatd{$\num{89.4}\pm\num{8.2}$} & \formatd{$\num{92.7}\pm\num{6.5}$} & \formata{$\num{98.4}\pm\num{1.7}$} & \formata{$\num{98.2}\pm\num{1.6}$}\\
 & $20\%$ (6) & \formatd{$\num{85.5}\pm\num{8.6}$} & \formatb{$\num{96.2}\pm\num{5.2}$} & \formata{$\num{99.1}\pm\num{1.6}$} & \formatb{$\num{97.9}\pm\num{1.8}$}\\
 & $50\%$ (17) & \formatd{$\num{96.5}\pm\num{4.8}$} & \formata{$\num{99.4}\pm\num{1.8}$} & \formata{$\num{99.4}\pm\num{1.8}$} & \formata{$\num{98.2}\pm\num{2.8}$}\\
\midrule
\multirow{6}{*}{\rotatebox[origin=c]{90}{\texttt{polblogs}}}
 & $\enspace 1\%$ (12) & \formatd{$\num{92.3}\pm\num{3.1}$} & \formatd{$\num{92.0}\pm\num{4.3}$} & \formata{$\num{95.6}\pm\num{0.2}$} & \formata{$\num{95.5}\pm\num{0.2}$}\\
 & $\enspace 2\%$ (24) & \formatd{$\num{93.1}\pm\num{1.8}$} & \formatd{$\num{94.1}\pm\num{1.4}$} & \formata{$\num{95.6}\pm\num{0.2}$} & \formata{$\num{95.5}\pm\num{0.2}$}\\
 & $\enspace 5\%$ (61) & \formatd{$\num{94.5}\pm\num{0.9}$} & \formatd{$\num{94.7}\pm\num{0.6}$} & \formata{$\num{95.6}\pm\num{0.2}$} & \formata{$\num{95.5}\pm\num{0.2}$}\\
 & $10\%$ (122) & \formatd{$\num{94.6}\pm\num{0.7}$} & \formatc{$\num{95.1}\pm\num{0.6}$} & \formata{$\num{95.6}\pm\num{0.2}$} & \formata{$\num{95.6}\pm\num{0.2}$}\\
 & $20\%$ (244) & \formatd{$\num{94.8}\pm\num{0.5}$} & \formatc{$\num{95.2}\pm\num{0.5}$} & \formata{$\num{95.6}\pm\num{0.3}$} & \formata{$\num{95.6}\pm\num{0.3}$}\\
 & $50\%$ (611) & \formatd{$\num{95.3}\pm\num{0.6}$} & \formatc{$\num{95.6}\pm\num{0.7}$} & \formata{$\num{95.8}\pm\num{0.8}$} & \formata{$\num{95.7}\pm\num{0.7}$}\\
\midrule
\multirow{6}{*}{\rotatebox[origin=c]{90}{\texttt{polbooks}}}
 & \emptyrow
 & $\enspace 2\%$ (2) & \formata{$\num{97.0}\pm\num{2.0}$} & \formata{$\num{97.8}\pm\num{0.8}$} & \formata{$\num{97.8}\pm\num{0.0}$} & \formata{$\num{97.8}\pm\num{0.0}$}\\
 & $\enspace 5\%$ (4) & \formata{$\num{97.8}\pm\num{1.0}$} & \formata{$\num{97.4}\pm\num{1.0}$} & \formata{$\num{97.7}\pm\num{0.0}$} & \formata{$\num{97.7}\pm\num{0.0}$}\\
 & $10\%$ (9) & \formata{$\num{97.5}\pm\num{1.7}$} & \formata{$\num{97.5}\pm\num{0.7}$} & \formata{$\num{97.7}\pm\num{0.3}$} & \formata{$\num{97.7}\pm\num{0.3}$}\\
 & $20\%$ (18) & \formata{$\num{97.8}\pm\num{1.3}$} & \formata{$\num{97.4}\pm\num{0.6}$} & \formata{$\num{97.5}\pm\num{0.5}$} & \formata{$\num{97.5}\pm\num{0.5}$}\\
 & $50\%$ (46) & \formata{$\num{97.8}\pm\num{1.7}$} & \formata{$\num{97.4}\pm\num{1.7}$} & \formata{$\num{97.4}\pm\num{1.7}$} & \formata{$\num{97.4}\pm\num{1.7}$}\\
\midrule
\multirow{6}{*}{\rotatebox[origin=c]{90}{\texttt{synth}}}
 & $\enspace 1\%$ (3) & \formata{$\num{79.7}\pm\num{13.5}$} & \formata{$\num{86.4}\pm\num{11.8}$} & \formata{$\num{87.0}\pm\num{12.9}$} & \formata{$\num{85.5}\pm\num{12.6}$}\\
 & $\enspace 2\%$ (6) & \formatd{$\num{81.8}\pm\num{9.2}$} & \formata{$\num{100.0}\pm\num{0.0}$} & \formatb{$\num{91.3}\pm\num{11.3}$} & \formatb{$\num{90.8}\pm\num{11.7}$}\\
 & $\enspace 5\%$ (15) & \formatd{$\num{88.2}\pm\num{8.5}$} & \formata{$\num{100.0}\pm\num{0.0}$} & \formatb{$\num{94.3}\pm\num{8.9}$} & \formatd{$\num{92.1}\pm\num{10.2}$}\\
 & $10\%$ (30) & \formatd{$\num{93.4}\pm\num{5.3}$} & \formata{$\num{100.0}\pm\num{0.0}$} & \formatb{$\num{98.0}\pm\num{4.2}$} & \formatd{$\num{96.1}\pm\num{6.8}$}\\
 & $20\%$ (60) & \formatd{$\num{97.9}\pm\num{1.8}$} & \formata{$\num{100.0}\pm\num{0.0}$} & \formatb{$\num{99.6}\pm\num{0.6}$} & \formatd{$\num{98.7}\pm\num{2.7}$}\\
 & $50\%$ (150) & \formatd{$\num{99.6}\pm\num{0.5}$} & \formata{$\num{100.0}\pm\num{0.0}$} & \formata{$\num{100.0}\pm\num{0.1}$} & \formatd{$\num{99.5}\pm\num{0.5}$}\\
  \bottomrule
 \end{tabular}
\end{mtable}

\subsubsection*{Dependence on the Tuning Parameter}

As mentioned before, for the smallest training sets used here, some of them composed by only two labelled nodes, it is impossible to perform a validation procedure.
To analyse the dependence of \sreg{}, \sbel{} and \srob{} on their tuning parameters, \ref{FigTuning} shows the evolution of the average test accuracy, both for the smallest and largest training sizes. The proposed \srob{} has the most stable behaviour, although as expected it sometimes drops near the critical value $\gamma = \eval1$. Nevertheless, this should be the easiest model to tune. \sreg{} shows also a quite smooth dependence, but with a sigmoid shape, where the maximum tends to be located in a narrow region at the middle. Finally, \sbel{} (the model comparable to \srob{} in terms of accuracy) presents the sharpest plot with large changes in the first steps, and hence it is expected to be more difficult to tune.

{
 \renewcommand{\showlegend}[1]{[\,\raisebox{1.5pt}{\plotline{Tuning}{#1}}\,]}
 \begin{mfigure}{\label{FigTuning} Comparison of the accuracy with respect to the different tuning parameters, for the smallest and largest training sets, and for the six datasets. \\ Legend: \showlegend{0}~\sreg{}; \showlegend{1}~\sbel{}; \showlegend{2}~\srob{}.}
   \begin{tabular}{r@{\,}l}
    \multicolumn{1}{r}{\makebox[0.405\textwidth][c]{\textbf{\num{1}\% of Labels}}} & \multicolumn{1}{l}{\makebox[0.405\textwidth][c]{\textbf{\num{50}\% of Labels}}} \\[1pt]
    {\includeg{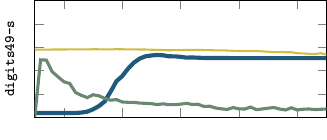}} & \includeg{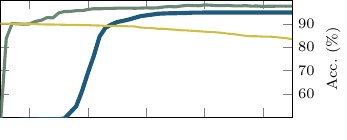}\\
    {\includeg{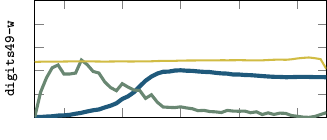}} & \includeg{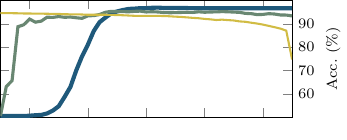}\\
    {\includeg{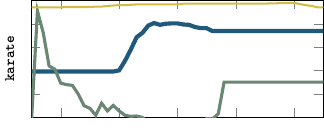}} & \includeg{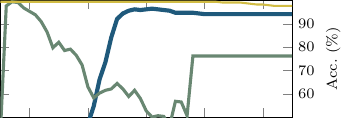}\\
    {\includeg{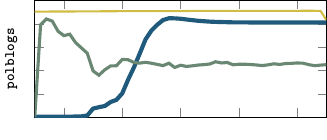}} & \includeg{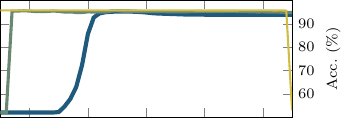}\\
    {\includeg{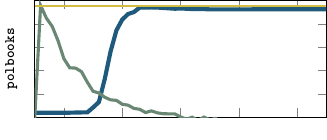}} & \includeg{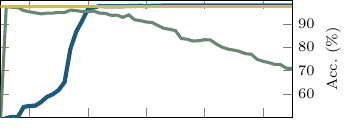}\\
    {\includeg{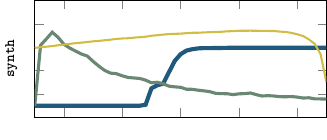}} & \includeg{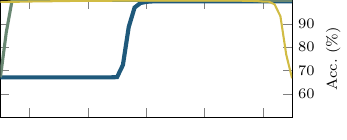}\\
  \includegraphics{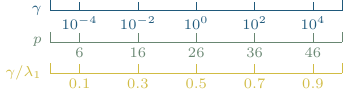} & \includegraphics{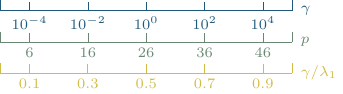}
   \end{tabular}
 \end{mfigure}
}

\subsection{Out-of-Sample Extension}

This experiment illustrates the out-of-sample extension, by comparing the accuracy of a model built using all the available graph and a model which is built with a smaller subgraph and then extended to the remaining nodes.

In particular, the dataset used is based on \texttt{digits49-w}, that is, a weighted graph representing the handwritten digits $4$ and $9$, but in this case the Gaussian kernel has a broader bandwidth of $\sigma = \num{5}$, so that the resulting model can be extended to all the patterns. Moreover, the total number of nodes is increased to \num{1000} (\num{500} of each class).
The number of labelled nodes is fixed to \num{10} (\num{5} of each class), whereas the size of the subgraph used to build the \srobf{} model is varied from \num{20} (\num{10} labelled and \num{10} unlabelled nodes) to \num{1000} (all the graph, \num{10} labelled and \num{990} unlabelled nodes). Once the \srobf{} model is built, the prediction is extended to the unlabelled nodes (both those used to build the model and those out of the initial subgraph, thanks to the out-of-sample extension), and the accuracy is measured.
The experiment is repeated \num{20} times to average the results, where the patterns are shuffled so that both the labelled nodes and the subgraph change.

The results are depicted in \ref{FigOOSEvo}, where the accuracy of the \srobf{} is shown as a function of the size of the subgraph used to built the model. As a baseline, an \srobf{} model using all the graph is also plotted. The solid lines represent the average accuracy, and the striped regions the areas between the minimum and maximum.
We can see that with a relatively small subgraph (of around \num{300} nodes) the subgraph model is comparable to the complete model (indeed, there is no statistically significant difference\footnote{Using a Wilcoxon signed rank test for zero median, with a significance level of \num{5}\%.} from iteration \num{179} on), so we can conclude that, in this particular case, the out-of-sample extension is working quite well.

{
 \renewcommand{\showlegend}[1]{$\left [ \begin{minipage}{22pt}\centering \plotline{OOSEvo}{#1} \\ \plotpattern{OOSEvo}{#1} \end{minipage} \right ]$}
 \begin{mfigure}{\label{FigOOSEvo} Comparison of the accuracy with respect to the size of the subgraph of an out-of-sample extended model and a model built using the complete graph. \\ Legend: \showlegend{0}~\srobf{} model (complete); \showlegend{1}~\srobf{} model (out-of-sample).}
  \includeg{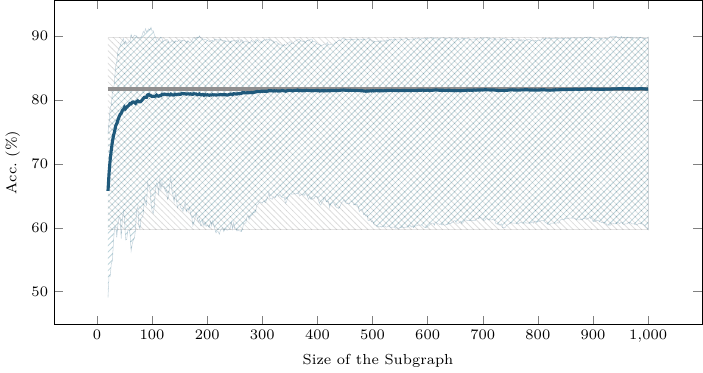}
 \end{mfigure}
}


\section{Conclusions}
\label{SecConclusions}

Starting from basic spectral graph theory, a novel classification method applicable to both semi-supervised classification and graph data classification has been derived in the framework of manifold learning, namely Robust Graph Classification (\srob{}). The method has a clear interpretation in terms of loss functions and regularization.
Noticeably, even though the loss function is concave, we have stated the conditions so that the optimization problem is convex. A simple algorithm to solve this problem has been proposed, which only requires to solve a linear system. The results of the method on artificial and real data show that \srob{} is indeed more robust to the presence of wrongly labelled data points, and it is also particularly well-suited when the number of available labels is small. Moreover, an out-of-sample extension for this model is proposed, which allows to extend the initial model to points out of the graph.

As further work, we intend to study with more detail the possibilities of the concave loss functions in supervised problems, bounding the solutions using either regularization terms or other alternative mechanisms.
Regarding the selection of $\gamma$, according to our results the predictions of \srob{} are quite stable with respect to changes in $\gamma$ in an interval containing the best parameter value. Hence, it seems that a stability criterion could be useful to tune $\gamma$.
On the other side, \srob{} could be applied to large-scale graph data, where the out-of-sample extension could play a crucial role in order to make the optimization problem affordable.
Moreover, the out-of-sample extension potentially allows the proposed \srob{} method to be used in completely supervised problems, and compared with other classification methods.

\begin{acknowledgements}
 The authors would like to thank the following organizations.
 \begin{itemize*}
  \item EU: The research leading to these results has received funding from the European Research Council under the European Union's Seventh Framework Programme (FP7/2007-2013) / ERC AdG A-DATADRIVE-B (290923). This paper reflects only the authors' views, the Union is not liable for any use that may be made of the contained information.
  \item Research Council KUL: GOA/10/09 MaNet, CoE PFV/10/002 (OPTEC), BIL12/11T; PhD/Postdoc grants.
  \item Flemish Government:
  \begin{itemize*}
   \item FWO: G.0377.12 (Structured systems), G.088114N (Tensor based data similarity); PhD/Postdoc grants.
   \item IWT: SBO POM (100031); PhD/Postdoc grants.
  \end{itemize*}
  \item iMinds Medical Information Technologies SBO 2014.
  \item Belgian Federal Science Policy Office: IUAP P7/19 (DYSCO, Dynamical systems, control and optimization, 2012-2017).
  \item Fundaci\'on BBVA: project FACIL--Ayudas Fundaci\'on BBVA a Equipos de Investigaci\'on Cient\'ifica 2016.
  \item UAM--ADIC Chair for Data Science and Machine Learning.
  \item Concerted Research Action (ARC) programme supported by the Federation Wallonia-Brussels (contract ARC 14/19-060 on Mining and Optimization of Big Data Models).
 \end{itemize*}
\end{acknowledgements}



\end{document}